\documentclass[12pt]{article}

\usepackage{fullpage,parskip,mdframed}

\usepackage{titlesec}

\usepackage{tikz}
\usetikzlibrary{decorations.pathmorphing}

\usepackage{pgfplots}
\pgfplotsset{compat=1.17}
\usepgfplotslibrary{fillbetween}

\usepackage{caption}
\usepackage[colorlinks]{hyperref}
\hypersetup{
  colorlinks=true,
  linkcolor=blue,
  citecolor=blue,
  urlcolor=red,
}

\usepackage{amsthm,amsmath,amssymb,mathtools,mathrsfs}

\usepackage{bm,upgreek}

\usepackage{physics,complexity}

\usepackage[ruled,vlined,linesnumbered]{algorithm2e}
\renewcommand{\R}{\mathbb{R}}
\newcommand{\N}{\mathbb{N}}

\renewcommand{\C}{\mathbb{C}}

\newcommand{\bzero}{\mathbf{0}}
\newcommand{\bone}{\mathbf{1}}

\newcommand{\be}{\mathbf{e}}
\newcommand{\bx}{\mathbf{x}}
\newcommand{\by}{\mathbf{y}}

\newcommand{\bu}{\mathbf{u}}
\newcommand{\bv}{\mathbf{v}}
\newcommand{\bw}{\mathbf{w}}
\newcommand{\bp}{\mathbf{p}}

\newcommand{\bA}{\mathbf{A}}
\newcommand{\bB}{\mathbf{B}}

\newcommand{\bP}{\mathbf{P}}

\newcommand{\bT}{\mathbf{T}}

\newcommand{\btheta}{\bm{\uptheta}}
\newcommand{\bpsi}{\bm{\uppsi}}

\newcommand{\bPsi}{\bm{\Psi}}

\newcommand{\calA}{\mathcal{A}}
\newcommand{\calB}{\mathcal{B}}

\newcommand{\calD}{\mathcal{D}}

\newcommand{\calU}{\mathcal{U}}

\newcommand{\ii}{\mathrm{i}}

\newcommand{\precompute}{\mathsf{PRECOMPUTE}}
\newcommand{\predict}{\mathsf{PREDICT}}
\newcommand{\size}{\mathsf{size}}

\newcommand{\mom}{\mathsf{MOM}}

\newcommand{\median}{\mathrm{median}}
\renewcommand{\poly}{\mathrm{poly}}
\renewcommand{\polylog}{\mathrm{polylog}}
\newcommand{\sym}{\Pi_{\mathrm{sym}}}
\newcommand{\opnorm}[1]{\norm{#1}_{\mathrm{op}}}
\newcommand{\fnorm}[1]{\norm{#1}_{\mathrm{F}}}

\DeclarePairedDelimiterX{\inner}[2]{\langle}{\rangle}{#1, #2}
\let\E\relax
\DeclareMathOperator*{\E}{\mathbb{E}}
\DeclareMathOperator*{\Var}{\mathsf{Var}}
\newcommand{\dfn}{\coloneqq}

\newtheorem{theorem}{Theorem}
\newtheorem{lemma}{Lemma}

\newtheorem{proposition}[lemma]{Proposition}
\newtheorem*{remark}{Remark}
\newtheorem*{theorem*}{Theorem}
\newtheorem*{fact*}{Fact}

\newtheorem{definition}{Definition}

\mdfdefinestyle{proofline}{
  leftline=true,
  rightline=false,
  topline=false,
  bottomline=false,
  linecolor=black!25,
  linewidth=1.1pt,
  skipabove=\medskipamount,
  skipbelow=\medskipamount,
  innerleftmargin=8pt,
  innerrightmargin=0pt,
}

\surroundwithmdframed[style=proofline]{proof}

\DeclareCaptionType{scheme}[Scheme][List of Schemes]
\makeatletter
\providecommand*{\toclevel@scheme}{0}
\makeatother

\SetKwInput{KwIn}{input}
\SetKwInput{KwOut}{output}
\SetKwInput{KwData}{data}
\newcommand{\AlgoMetaSep}{\vspace{8pt}}

\title{How to sketch a learning algorithm}
\author{Sam Gunn \\
    UC Berkeley \\
    \texttt{gunn@berkeley.edu}}

\begin{document}

\maketitle

\begin{abstract}
    How does the choice of training data influence an AI model?
    This broad question is of central importance to interpretability, privacy, and basic science.
    At its technical core is the data deletion problem: after a reasonable amount of precomputation, quickly predict how the model would behave in a given situation if a given subset of training data had been excluded from the learning algorithm.

    We present a data deletion scheme capable of predicting model outputs with vanishing error $\varepsilon$ and failure probability $\delta$ in the deep learning setting.
    Our precomputation and prediction algorithms are only $\tilde{O}(\log(1/\delta)/\varepsilon^2)$ factors slower than regular training and inference, respectively.
    The storage requirements are those of $\tilde{O}(\log(1/\delta)/\varepsilon^2)$ models.

    Our proof is based on an assumption that we call \emph{stability}.
    In contrast to the assumptions made by prior work, stability appears to be fully compatible with learning powerful AI models.
    In support of this, we show that stability is satisfied in a minimal set of experiments with microgpt.
    Our code is available at \href{https://github.com/SamSpo1/microgpt-sketch}{microgpt-sketch}.

    At a technical level, our work is based on a new method for locally sketching an arithmetic circuit by computing higher-order derivatives in random complex directions.
    Forward-mode automatic differentiation allows cheap computation of these derivatives.
\end{abstract}

\newpage
{\small\tableofcontents}

% !TEX root = main.tex
\section{Introduction} \label{section:intro}

One of the most fundamental problems in AI is to understand model behavior as a function of the choice of training data.
Variations on this problem abound, going under names such as influence functions, data selection, datamodeling, and machine unlearning.
While the exact problem formulation depends on the application of interest, they are all addressed by \emph{data deletion}.\footnote{The term ``data deletion'' has multiple meanings in the literature. We use it because we feel it most accurately captures the problem we consider.}

A data deletion scheme consists of two algorithms:
\begin{itemize}
    \item \emph{precomputation}, which takes as input the description of a learning algorithm $\calA$ and a training dataset $T$, and outputs the trained model $\calA(T)$ together with some arbitrary auxiliary information; and
    \item \emph{prediction}, which takes as input the description of a function $\phi$ of the model parameters (we call $\phi$ a ``measurement'') together with a small subset $D$ of the training dataset, and uses the precomputed auxiliary information to quickly predict $\phi(\calA(T \setminus D))$: i.e., what the measurement outcome would have been \emph{if the subset $D$ had been excluded from the training dataset}.
\end{itemize}
To illustrate one of the many applications of such a protocol, consider the following scenario.
Suppose that the work of an artist is included in the training data for a new AI model.
This model later generates some content $x$ that appears related to the artist's work.
How can we determine whether the relationship is causal?

If we use a data deletion scheme, then we can quickly compute the probability that the model would have produced the given content if all of the artist's work were excluded from the set of training data.
All we need to do is search the training data for the artist's work $D$ and then query the prediction algorithm, using the measurement ``what is the probability that this model will produce $x$?'' and the deletion set $D$.

Although they would be extremely useful, data deletion schemes are notoriously hard to come by.
There currently \textbf{does not exist} a scheme capable of achieving arbitrary prediction error $\varepsilon$ for actual AI algorithms---in theory or in practice.
Instead, existing data deletion schemes universally rely on surrogate models or heuristics, and often only weakly correlate with the ground truth \cite{magic}.

In this work, we demonstrate that this situation may not be inevitable.
We introduce a new algorithmic technique which allows us to build a data deletion scheme that is directly applicable to realistic AI pipelines.
We prove that our scheme can achieve arbitrary prediction error $\varepsilon$ with arbitrary failure probability $\delta$, using only a $\tilde{O}(\log(1/\delta) / \varepsilon^2)$ overhead in each of training, inference, and storage.
Our proof is based on a new assumption about the learning algorithm which we call \emph{stability}.
We find that stability is practically reasonable using a small set of experiments with microgpt \cite{microgpt}.

Stability is an assumption about both the learning algorithm $\calA$ and the measurement $\phi$.
Roughly, it quantifies the sensitivity of $\phi(\calA(T))$ to small changes in the choice of data $T$.
The more robust the measurement is to changes in training data, the more stable it is.

Our data deletion scheme has the following properties:
\begin{itemize}
    \item precomputation is only a $\tilde{O}(\log(1/\delta)/\varepsilon^2)$ factor more expensive than training the model;
    \item predicting model behavior after deleting a set of data is only a $\tilde{O}(\log(1/\delta)/\varepsilon^2)$ factor more expensive than reading the deleted set of data and then evaluating the trained model once;
    \item the precomputed data can be stored in the space of $\tilde{O}(\log(1/\delta)/\varepsilon^2)$ trained models; and
    \item under the weakest form of our stability assumption, the prediction error is $O(\varepsilon)$ for any constant number of deletions $d$, while a stronger form of our stability assumption bounds the prediction error for arbitrary $d$.
\end{itemize}

Our techniques are based on a new method of locally sketching an arithmetic circuit, which may be of independent interest.

% !TEX root = main.tex
\subsection{Technical overview} \label{subsection:technical-overview}
In this section we explain our setting, assumption, and method in a more concise and intuitive format than the more technical sections which follow.
We include references to the relevant technical sections throughout.

\paragraph{The setting.}
Let us begin by describing the data deletion problem in more detail.
Our guiding principle is that we do not want our specification to exclude real-world AI pipelines, but we still want it to be completely rigorous.

For the purposes of this introduction, we restrict our attention to stochastic gradient descent on $n$ training examples with single-example batches.
This restriction is entirely pedagogical, and the technical sections use general arithmetic circuits.
See Section~\ref{subsection:arithmetic-circuits} for details on arithmetic circuits and how they relate to machine learning.

Following \cite{datamodels,magic}, we define the \emph{learning algorithm} $\calA$ as the following function mapping a vector of per-example \emph{downweights}\footnote{Whereas \cite{datamodels,magic} upweight examples, for our purposes it is cleaner to downweight them.} $\bw = (w_1, \dots, w_n) \in \R^n$ to trained model parameters $\btheta_n \in \R^p$.

\begin{samepage}
\begin{list}{}{\setlength{\leftmargin}{2em}}
    \item[] $\calA$ on input $\bw \in \R^n$:
    \begin{enumerate}
        \item Initialize $\btheta_0 \in \R^p$ using hard-coded randomness.
        \item For $i = 1, \dots, n$:
        \begin{enumerate}
            \item Let $\ell_i : \R^p \to \R$ be the loss on example $i$.
            \item Compute $\btheta_i \gets \btheta_{i-1} - (1 - w_i)\,\grad \ell_i(\btheta_{i-1})$.
        \end{enumerate}
        \item Output $\btheta_n$.
    \end{enumerate}
\end{list}
\end{samepage}
Observe that the downweights specify an extent to which data is ignored, so that if $\bone_D$ is the indicator vector for the set $D \subseteq [n]$ then $\calA(\bone_D)$ is the model trained on data indices $[n] \setminus D$.

The model can be probed using a \emph{measurement function} $\phi : \R^p \to \R$ mapping model parameters to an outcome.
For instance, $\phi$ can be the output of the model on some example (which should be hard-coded into $\phi$).
We assume that the measurement $\phi$ and loss functions $\ell_i$ are described by arithmetic circuits mapping $p$ inputs to 1 output.
By the Baur-Strassen theorem, this implies that $\calA$ and $\phi \circ \calA$ are also arithmetic circuits.

A data deletion scheme decomposes the evaluation of the circuit $f = \phi \circ \calA$ on input $\bone_D$ into two parts: slow precomputation depending only on $\calA$, and fast prediction depending on $\phi$ and $D$.
We say that a data deletion scheme has error $\varepsilon$ for a given $\phi$ and $D$ if the prediction is $\varepsilon$-close to $f(\bone_D)$.
Throughout this work, we require the prediction to be accurate even for a fixed choice of all the randomness used in training and inference (such as the initialization, batch ordering, and latent noise).
See the discussion on ``single-model'' data attribution in \cite[Section 2.1]{magic} for an explanation of why this is a more useful, and stringent, definition.

We will require that $f = \phi \circ \calA$ satisfies a condition we call ``stability,'' which we explain next.
See Section~\ref{subsection:stable-circuits} for a more thorough exposition on stability.

\paragraph{The stability assumption.}
Consider the Taylor expansion of $f : \R^n \to \R$ around $\bzero$,
\begin{equation} \label{equation:taylor}
    f(\bw) = \sum_{r \ge 0} \inner{\bT^{(r)}}{\bw^{\otimes r}},
\end{equation}
where $\inner{\cdot}{\cdot}$ denotes the inner product (after flattening tensors) and $\bT^{(r)} \in (\R^n)^{\otimes r}$ is the $r$th Taylor coefficient\footnote{The re-scaled coefficient $r! \, \bT^{(r)}$ consists of all $r$th-order mixed partial derivatives of $f$ at $\bzero$.} of $f$.
We say that $f$ is \emph{stable} if the terms in the Taylor expansion decay as
\begin{equation} \label{equation:stability}
    \fnorm{\bT^{(r)}} \le 2^{-\omega(r)},
\end{equation}
where $\fnorm{\cdot}$ denotes the Frobenius norm (i.e., the 2-norm of the flattened tensor) and $\omega(r)$ denotes growth that is super-linear in $r$.
For the purposes of this overview, we require Inequality~(\ref{equation:stability}) to hold for some fixed $\omega(r)$ function, even as the amount of training data $n$ and the size of the model $p$ grow.
In the technical sections we will use a weaker definition of stability which is not asymptotic at all: It only requires exponential decay in $r$, where the base of the exponent determines the number of deletions we can handle.

It turns out that \textbf{Inequality~(\ref{equation:stability}) is the only assumption we need}.
In Section~\ref{subsection:microgpt}, we find that Inequality~(\ref{equation:stability}) is satisfied in minimal experiments with microgpt \cite{microgpt}.

\begin{theorem*}[Special case of Theorem~\ref{theorem:deletion}]
    For any $\varepsilon, \delta > 0$, there exists a data deletion scheme such that the following holds.
    For any learning algorithm $\calA$ taking $n$ training examples to $p$ model parameters,
    \begin{itemize}
        \item precomputation requires $\tilde{O}(\size(\calA) \, \log(1/\delta) / \varepsilon^2)$ operations,
        \item the required storage is\footnote{Assuming pseudorandom generators exist, the dependence on $n$ can be dropped.} $\tilde{O}((n + p) \log(1/\delta) / \varepsilon^2)$, and
    \end{itemize}
    for any measurement $\phi$ and any set $D \subseteq [n]$ of $O(1)$ examples,
    \begin{itemize}
        \item predicting $(\phi \circ \calA)(\bone_D)$ requires $\tilde{O}(\size(\phi) \, \log(1/\delta) / \varepsilon^2)$ operations, and
        \item the prediction has error $O(\varepsilon)$ with probability $1 - \delta$ assuming $f = \phi \circ \calA$ is stable in the sense of Inequality~(\ref{equation:stability}),
    \end{itemize}
    where $\tilde{O}$ hides $\polylog(1/\varepsilon)$ factors.
\end{theorem*}

In the body of the paper, we use a more general and precise definition of stability that implies a bound on the approximation error under any number of deletions.
This definition is given in Section~\ref{subsection:stable-circuits}, where we also explain how stability can be viewed as a combination of an intuitive continuity property and a bound on the ``stable rank'' of the derivatives.

The remainder of this section is an overview of our data deletion scheme.

\paragraph{The approach.}
For now, let us suppose that we know $\phi$ during precomputation.
Our goal is then to compute a local sketch of $f$ that enables fast approximation of $f(\bone_D)$ for any small set $D$, assuming that $f$ is stable.
At the end of this section we will explain how this extends to our data deletion scheme (where the function $f$ is not fully known during precomputation).
Our data deletion is presented formally in Section~\ref{section:data-deletion}.

We begin by observing that we can truncate Equation~(\ref{equation:taylor}) to finite order without incurring too much error.
For any constant-sized set $D$, Inequality~(\ref{equation:stability}) implies that the $s$th order Taylor approximation to $f(\bone_D)$ has error $O(\varepsilon)$ for $s = \log(1/\varepsilon)$.

A natural approach to data deletion is then to have precomputation compute and store $\bT^{(r)}$ for $r \le s$.
Then given $\bT^{(r)}$ and $D$, prediction can compute $\inner{\bT^{(r)}}{\bone_D^{\otimes r}}$ in time $O(\abs{D}^r)$, so it can compute the $s$th order Taylor approximation to $f(\bone_D)$ in time $O(\abs{D}^s)$.
This is faster than re-training for small enough $D$ and $s$.
However, the precomputed data itself has size $n^s$.
Since $n$ can be in the trillions for large AI models, this is prohibitively expensive even for $s = 2$.

\paragraph{Sketching by random projections.}
The storage issue is easy to solve using random projections: Instead of having precomputation store the entire tensor $\bT^{(r)}$, it can store $\inner{\bT^{(r)}}{\bPsi_i^{(r)}}$ for many random unit vectors $\bPsi_1^{(r)}, \dots, \bPsi_k^{(r)} \in (\R^n)^{\otimes r}$, where the coordinates of each $\bPsi_i^{(r)}$ are selected with a pseudorandom function.
Since
\[
    n^r \, \E_{\bPsi}\left[\inner{\bT^{(r)}}{\bPsi} \cdot \inner{\bPsi}{\bone_D^{\otimes r}}\right] = \inner{\bT^{(r)}}{\bone_D^{\otimes r}},
\]
our prediction algorithm can use the stored values $\inner{\bT^{(r)}}{\bPsi_i^{(r)}}$ to compute an unbiased estimate for $\inner{\bT^{(r)}}{\bone_D^{\otimes r}}$.
Furthermore, the Johnson-Lindenstrauss lemma \cite{JL} implies a bound on the variance of roughly $\frac{\abs{D}^r}{k} \, \fnorm{\bT^{(r)}}^2$.
Assuming Inequality~(\ref{equation:stability}), this is only $o(1/k)$ for any constant-sized $D$.

Of course, the main issue persists: Precomputation still needs to produce the entire tensor $\bT^{(r)}$ in order to compute $\inner{\bT^{(r)}}{\bPsi_i^{(r)}}$.
What we need is to compute a random projection \emph{implicitly}, without ever materializing $\bT^{(r)}$.

\paragraph{Sketching by differentiation.}
It turns out that if $f$ is an arithmetic circuit, there is a very simple algorithm for implicitly computing random projections of the Taylor coefficients $\bT^{(r)}$ for all $r \le s$:
\begin{center}
\begin{minipage}{0.75\textwidth}
\centering
    \it
    Sample a random unit vector $\bpsi \in \C^n$ and evaluate $f$ on input $z \, \bpsi$ over the polynomial ring $R \dfn \C[z] / (z^{s+1})$.
\end{minipage}
\end{center}
If the circuit for $f$ includes any non-polynomial unary gate, it may simply be replaced by the Taylor expansion around the constant term of the input polynomial.
Whereas it takes time $O(\size(f))$ to evaluate $f$ over $\R$ or $\C$, it takes $O(\size(f) \, s \log s)$ to evaluate $f$ over the polynomial ring $R$; see Section~\ref{subsection:forward-ad} for details.

Substituting $z \, \bpsi$ for $\bw$ in Equation~(\ref{equation:taylor}),
\[
    f(z \, \bpsi) = \sum_{r=0}^s \inner{\bT^{(r)}}{\bpsi^{\otimes r}} \, z^r
\]
in the polynomial ring $R$.
Therefore we obtain random projections of the $\bT^{(r)}$ by simply computing $f(z \, \bpsi)$ and reading off the polynomial coefficients!
The coefficient on $z^r$ in $f(z \, \bpsi)$ is precisely $\inner{\bT^{(r)}}{\bpsi^{\otimes r}}$.

Of course, this algorithm does \emph{not} give us projections in uniformly random directions $\bPsi \in (\R^n)^{\otimes r}$.
Instead we have projections in directions of the form $\bpsi^{\otimes r}$ for random unit vectors $\bpsi \in \C^n$.
These turn out to be good enough.

The key fact is that for any symmetric tensors $\bA, \bB \in (\R^n)^{\otimes r}$, if $\bpsi \in \C^n$ is a random complex unit vector, then
\begin{equation} \label{equation:symmetric-subspace}
    \binom{n+r-1}{r} \, \E_{\bpsi}\left[\inner{\bA}{\bpsi^{\otimes r}} \, \inner{\bpsi^{\otimes r}}{\bB}\right] = \inner{\bA}{\bB}.
\end{equation}
We prove Equation~(\ref{equation:symmetric-subspace}) in Section~\ref{subsection:symmetric-subspace}, where we also show that the variance is at most $4^r \, \fnorm{\bA}^2 \, \fnorm{\bB}^2$.

Substituting $\bA \dfn \bT^{(r)}$ and $\bB \dfn \bone_D^{\otimes r}$, this allows our prediction algorithm to estimate $\inner{\bT^{(r)}}{\bone_D^{\otimes r}}$ for a given $D$ by making use of the precomputed values $\inner{\bT^{(r)}}{\bpsi_i^{\otimes r}}$.
That is, our prediction algorithm can use
\[
    \frac{1}{k} \, \sum_{i=1}^k \binom{n+r-1}{r} \, \inner{\bT^{(r)}}{\bpsi_i^{\otimes r}} \, \inner{\bpsi_i}{\bone_D}^r
\]
as an unbiased estimate for $\inner{\bT^{(r)}}{\bone_D^{\otimes r}}$.
The variance is bounded by $\frac{(4 \abs{D})^r}{k} \, \fnorm{\bT^{(r)}}^2$, which is $o(1/k)$ assuming Inequality~(\ref{equation:stability}).
By setting $k \approx 1/\varepsilon^2$, we will get error of roughly $O(\varepsilon)$.

Finally, our prediction algorithm estimates $f(\bone_D)$ as
\begin{align*}
    f(\bone_D) &\approx \sum_{r=0}^s \inner{\bT^{(r)}}{\bone_D^{\otimes r}} \\
    &\approx \sum_{r=0}^s \frac{1}{k} \sum_{i=1}^k \binom{n+r-1}{r} \, \inner{\bT^{(r)}}{\bpsi_i^{\otimes r}} \, \inner{\bpsi_i}{\bone_D}^r.
\end{align*}
As noted earlier, Inequality~(\ref{equation:stability}) implies that the first approximation incurs an error of $O(\varepsilon)$ when $s = \log(1/\varepsilon)$.
And as just discussed, another use of Inequality~(\ref{equation:stability}) implies that the second approximation incurs an additional error of $O(\varepsilon)$ for $k \approx 1/\varepsilon^2$.
See Theorem~\ref{theorem:deletion} for a full derivation and the technical result.

\paragraph{Upgrading to data deletion.}
So far we have assumed that precomputation knows $\phi$, which is not allowed in data deletion.
In our actual scheme, precomputation works identically except that $f$ is replaced by $\calA$.
Even though $\calA$ has many outputs, this does not increase the computational complexity: Precomputation can still be done in time $O(\size(\calA) \, s \log s)$ by simply evaluating each $\calA(z \, \bpsi_i)$ over the polynomial ring $R \dfn \C[z] / (z^{s+1})$.

Prediction, on the other hand, becomes slightly more complicated.
While it would be natural to simply compute the analogous approximation to $\calA(\bone_D)$ as was done for $f(\bone_D)$, and to then apply $\phi$ to this approximation, this is not guaranteed to be accurate by stability alone.
For instance, imagine that $\calA$ performs some arbitrarily unstable computation, which is then undone by $\phi$ so that $f = \phi \circ \calA$ is a constant function.
Then $f$ is highly stable, but any approximation error on $\calA(\bone_D)$ translates to large error on the output of $\phi$.

Instead, our prediction algorithm will first compute $\phi$ on the ring elements output by our sketch of $\calA$, thereby obtaining \emph{the exact same sketch} that we would have obtained if we had applied precomputation directly to $f$.
We can then evaluate this sketch in the same way as before.

\paragraph{Minor differences from the main body.}
We have glossed over several subtleties for the purposes of this introduction, creating differences from the main body of the paper.
To minimize confusion, we briefly explain the main differences now.

First, the body considers the learning and measurement algorithms as arbitrary arithmetic circuits; our restriction to stochastic gradient descent in the introduction is not important.

Second, we find it convenient to work over the complex numbers $\C$ rather than $\R$.
Since our model for AI computations is simply arithmetic circuits, this does not cost us any generality.
In any case, as explained in Section~\ref{subsection:arithmetic-circuits}, this choice is not practically important and does not hinder our implementation at \href{https://github.com/SamSpo1/microgpt-sketch}{microgpt-sketch}.

Third, the definition of stability presented in Inequality~(\ref{equation:stability}) is less general than the one we actually use in the body.
The real definition, presented in Section~\ref{subsection:stable-circuits}, is more conducive to proving quantitative bounds on the error of our schemes.

Fourth, we are unable to show that an average of estimates from Equation~(\ref{equation:symmetric-subspace}) concentrates well.
We would therefore obtain a poor bound on the failure probability if we used the scheme from this introduction.
Instead, we use a median-of-means estimator in our prediction algorithm (Algorithm~\ref{algorithm:prediction}).

% !TEX root = main.tex
\subsection{Related work} \label{subsection:related-work}

Variations on the data deletion problem abound, and we only survey some of the most relevant work here.
Our work is most similar to the line of work on \emph{datamodeling} or \emph{predictive data attribution}, although there are several key differences which we outline below.
For an introduction to the area, we recommend the tutorial of \cite{icml-tutorial}.

A \emph{surrogate model} is an idealized model of a system which is used in place of the actual system.
Every prior work with provable guarantees for data deletion in the deep learning setting relies on a surrogate model.
This precludes a vanishing prediction error, as any method that relies on a surrogate model necessarily incurs the fixed approximation error of the surrogate model itself.
In contrast, our stability assumption is not a surrogate model: It appears to be fully compatible with learning powerful models on large datasets.

\paragraph{The additivity assumption.}
Perhaps the most successful surrogate model for answering questions of data influence has been the \emph{additivity} assumption.
This assumption says that the function $f$ mapping the list of data weights to a model prediction (on a fixed test query) is linear.
In other words, it is assumed that the influence of a set of training data on any model prediction is the sum of the influences of the individual training examples.
For instance, the work on \emph{datamodeling} \cite{datamodels,trak,magic} and \emph{influence functions} \cite{influence-functions-1,influence-functions-2,influence-functions-3} uses the additivity assumption.

In reality, it is of course understood that additivity does not hold for real models.
The hope is rather that the additivity assumption does not incur too large an error, at least for small perturbations in the data weights.
There are theoretical reasons to believe that the additivity assumption should incur bounded error for random perturbations \cite{datamodels,fourier-influence-functions}, but in practice it is usually necessary to significantly alter learning algorithms in order to keep this error small \cite{magic}.
Moreover, random perturbations may not capture many of the most important use cases such as removing many closely-related examples.

Designing data deletion schemes with the additivity assumption has several major drawbacks:
\begin{itemize}
    \item the prediction error becomes unreasonable past a certain number of deletions;
    \item there is no way to reduce the prediction error by expending more computation;
    \item even if additivity is a good approximation under random deletions, it can be a poor approximation for worst-case deletions.
\end{itemize}
We therefore view our departure from the additivity assumption as an essential aspect of this work.

However, we emphasize that \emph{even assuming additivity}, our scheme greatly improves on prior work.
The strongest result in the additivity model is the work of \cite{magic}, which exactly computes counterfactuals assuming additivity.
Accepting additivity, the main drawback of that work is that it is necessary to know the measurement function ahead of time during precomputation.
Our work yields the first scheme in the deep learning setting to both (a) not require foreknowledge of the measurement function during precomputation and (b) guarantee a provable bound on the counterfactual prediction error.

\paragraph{Post hoc certified unlearning.}
It is common for machine unlearning schemes to make no reference to the algorithm used to train the model.
For instance, one method is based on performing gradient ascent on the loss for the set of data to be deleted \cite{neurips-comp}.
These methods can be applied to ``delete'' even data that was never seen during training of the model.
We discuss this approach in more detail in Appendix~\ref{section:post-hoc-unlearning}.

\paragraph{Fourier analysis.}
As mentioned in the paragraph on the additivity assumption, \cite{fourier-influence-functions} argue that the additivity assumption should be a good approximation for the influence of random subsets of training data.
They adopt the ``discrete influence'' perspective, where one views the datamodeling function as being restricted to the hypercube, $f : \{0,1\}^n \to \R$, allowing them to use ideas from the analysis of Boolean functions.
It would be interesting to see whether similar arguments can be used to shed light on our stability assumption.
It would also be interesting to see whether our techniques can be extended to the discrete influence setting.
This could potentially enable data deletion using a stability assumption related to the Fourier spectrum of the discrete influence function.

\paragraph{Other work.}
Since we are not focused on any particular application of data deletion, we do not discuss considerations such as privacy concerns when applying data deletion as a solution to machine unlearning.
We refer the interested reader to \cite{sanjam,aloni}.
We caution that data deletion by itself is not a complete solution to machine unlearning, and na\"{\i}ve applications of data deletion can in some cases actually \emph{harm} privacy \cite{mul-jeopardizes-privacy,privacy-onion,unlearning-harm}.

There is also work on provable data deletion for particular algorithms such as $k$-means clustering \cite{valiant}, linear classifiers \cite{goldstein}, and linear or logistic regression \cite{regression-unlearning}.

% !TEX root = main.tex
\section{Preliminaries} \label{section:preliminaries}
This section has two parts: Section~\ref{subsection:notation}, which contains the notation and basic definitions used throughout the paper; and Section~\ref{subsection:arithmetic-circuits}, which explains the mathematical model we use to describe AI computations.

\subsection{Notation} \label{subsection:notation}
Let $\N$ denote the set of positive integers, $\R$ denote the real numbers, and $\C$ denote the complex numbers.
For $t \in \N$, let $[t] \dfn \{1, \dots, t\}$; for $u, v \in \R$, let $[u, v] \dfn \{w \in \R \mid u \le w \le v\}$.
We write $\sum_{r \ge i}$ to denote the sum over all integers $r \ge i$.
We use $n[r]$ to denote the dimension of the symmetric subspace of $(\C^n)^{\otimes r}$, i.e., $n[r] \dfn \binom{n+r-1}{r}$.
See Section~\ref{subsection:symmetric-subspace} for more background on the symmetric subspace.

\paragraph{Tensors and linear algebra.}
We use bold lowercase letters to denote vectors and bold uppercase letters to denote matrices, tensors, or other collections of vectors.
If $\bu \in \C^n$ is a vector, we write $u_i$ to denote the $i$th coordinate of $\bu$; we write $\bu_i$ to denote a vector in a collection of vectors $\bu_1, \dots, \bu_r$.

Let $\bpsi^T$ denote the transpose and $\bpsi^\dagger$ denote the conjugate transpose of $\bpsi \in \C^n$, and let $\inner{\bu}{\bv} \dfn \bu^\dagger \bv$ denote the inner product of vectors $\bu, \bv \in \C^n$.
If $\bu$ or $\bv$ is a tensor, it is flattened and interpreted as a vector for the purposes of computing inner products.
For a complex scalar, vector, matrix, or general tensor $\bT$, let $\overline{\bT}$ denote the complex conjugate.

For a vector $\bu \in \C^n$, let $\norm{\bu}$ denote the 2-norm.
For a tensor $\bT \in (\C^n)^{\otimes r}$, let $\fnorm{\bT} = \sqrt{\inner{\bT}{\bT}}$ denote the Frobenius norm and let $\opnorm{\bT}$ denote the operator norm defined as
\[
    \opnorm{\bT} \dfn \sup_{\substack{\bu_1, \dots, \bu_r \in \C^n \\ \norm{\bu_1} = \dots = \norm{\bu_r} = 1}} \abs{\inner{\bT}{\bu_1 \otimes \dots \otimes \bu_r}}.
\]
A \emph{random unit vector} in $\C^n$ (resp. $\R^n$) refers to a sample from the Haar distribution---i.e., the unique distribution which is invariant under the action of every unitary (resp. orthogonal) matrix---on unit vectors in $\C^n$ (resp. $\R^n$).

\paragraph{Calculus.}
For a function $f : \C \to \C^p$, let $\partial^r f(x) \in \C^p$ denote the vector of $r$th-order derivatives of $f$ evaluated at $x$.
For a function $f : \C^n \to \C^p$, let $\partial_i^r f(\bx) \in \C^p$ be the vector of $r$th order partial derivatives of $f$ with respect to the $i$th input, evaluated at $\bx$.
If $r$ is not specified then it is interpreted as $r = 1$.

\paragraph{Computation.}
We count real operations as constant time and we assume that real numbers can be stored exactly in constant space.
We further assume random access memory (RAM) with a word size large enough to look up any value in memory, so that arbitrary lookups take constant time.

\subsection{Arithmetic circuits and machine learning} \label{subsection:arithmetic-circuits}
We make the pedagogical decision to model machine learning computations---both training and inference---as arithmetic circuits over the complex numbers.
All of our results apply essentially identically to real computations including general gates such as ReLU, GELU, or Softmax; the technical details would just be more cumbersome.
While these gates may not be defined on complex inputs, our method only requires that they have local polynomial approximations around real inputs.
Indeed, we easily implemented the general case for microgpt (see Section~\ref{subsection:microgpt}).

Formally, an \emph{arithmetic circuit} with $n$ inputs and $p$ outputs is defined by a directed acyclic graph with labeled nodes.
Each node in the graph with in-degree zero is labeled by either an input index $i \in [n]$ or a complex number.
Each node with positive in-degree is labeled by either $+$ or $\times$.
The output of a gate is passed as input to all of its children.
We do not restrict the in-degree or out-degree.
Each node with out-degree zero is further labeled by an output index $i \in [p]$.
Whereas input indices may be re-used, each output index is a label for exactly one node.
For an input $\bx$ consisting of $n$ values $(x_1, \dots, x_n)$, the output $f(\bx)$ is computed by replacing each input-index node label $i \in [n]$ with the corresponding input value $x_i$, following the computation specified by the graph, and outputting the vector $\by$ specified by the output values at the output indices.
Note that the inputs $x_i$ can be elements of any $\C$-algebra, such as $\R$, $\C$, or $\C[z] / (z^5)$.

We define the size of $f$, denoted $\size(f)$, as the maximum of the number of inputs, the number of outputs, and the number of edges in the graph.
Since we count real operations as constant time, evaluating $f$ requires $O(\size(f))$ operations.

In the context of machine learning, note that if model evaluations can be described by arithmetic circuits acting on model parameters, then the Baur-Strassen theorem \cite[Theorem 9.10]{arithmetic-complexity} implies that gradient steps can also be represented by arithmetic circuits.
Therefore it is without further loss of generality to assume that training, in addition to inference, is an arithmetic circuit.

% !TEX root = main.tex
\section{The toolkit}
\label{section:toolkit}

This section explains the basic facts and tools that we will use to build our sketching and data deletion schemes.
In Section~\ref{subsection:forward-ad} we briefly recall the computational complexity of various kinds of differentiation.
In Section~\ref{subsection:symmetric-subspace} we explain some basic facts about the symmetric subspace which underlie our sketching and data deletion schemes.
In Section~\ref{subsection:stable-circuits} we give our definition of stability for arithmetic circuits, which we will later use to analyze our data deletion scheme.

% !TEX root = main.tex
\subsection{Forward-mode automatic differentiation} \label{subsection:forward-ad}
Reverse-mode automatic differentiation (i.e., backpropagation) can compute the gradient of an arithmetic circuit $f$ with \textbf{$n$ inputs and 1 output} in $O(\size(f))$ operations.
It does not extend to higher-order derivatives: There is no known $\size(f) \cdot \poly(s)$-time algorithm for computing the set of univariate higher-order derivatives $\{\partial_i^s f(\bx)\}_{i \in [n]}$ for $\bx \in \C^n$.
Reverse-mode automatic differentiation also suffers from a memory issue, requiring the storage of $\size(f)$ values from the forward pass.
While the memory burden can sometimes be reduced using a technique called rematerialization \cite{rematerialization,evaluating-derivatives,replay,magic}, this comes at the cost of increased computation.

In contrast, forward-mode automatic differentiation enables the computation of \emph{all derivatives} up to order $s$ of an arithmetic circuit $f$ with \textbf{1 input and $p$ outputs} in $O(\size(f) \, s \log s)$ operations, with only an $O(s \log s)$ overhead in memory.
Forward-mode automatic differentiation is extremely simple.
Instead of computing $f$ over $\C$ on an input value $x \in \C$, we compute $f$ over the ring $R \dfn \C[z]/(z^{s+1})$ on an input polynomial $x + z \in R$.
If the output polynomials are $a_1(z), \dots, a_p(z)$, then the coefficient on $z^r$ in $a_j(z)$ is exactly $(1/r!) \, \partial^r f_j(x)$, where $f_j(x)$ is the $j$th output coordinate of $f(x)$.
To compute a directional derivative---say, the derivative of $f$ at $\bx$ in direction $\bpsi$---we simply compute $f(\bx + z \bpsi)$ over $R$.
Addition gates take $O(s)$ operations per input and multiplication gates take $O(s \log s)$ per input using the fast Fourier transform, so the total complexity is $O(\size(f) \, s \log s)$.

Because forward-mode automatic differentiation can be interpreted simply as evaluation over the ring $R \dfn \C[z]/(z^{s+1})$, we will not use calculus terminology in the remainder of this work.
The only fact we use from this section is that we can evaluate $f$ over $R$ in time $O(\size(f) \, s \log s)$.

% !TEX root = main.tex
\subsection{The symmetric subspace} \label{subsection:symmetric-subspace}

In this section we present the facts we will require about the symmetric subspace.
For the reader interested in more background, we suggest the first section of the notes of Aram Harrow \cite{church}.\footnote{While those notes are intended for a quantum information audience, the reader only needs to tolerate bra-ket notation in order to follow the first section.}
The main purpose of this section is to prove Theorem~\ref{theorem:symmetric-subspace}.
This theorem is stated as Proposition 6 in \cite{church}, but we give a somewhat different proof which we find simpler and more direct.

The \emph{symmetric subspace} of $(\C^n)^{\otimes r}$, written $\vee^r \C^n$, is the linear subspace that is fixed under permutations of the $r$ tensor factors:
\[
    \vee^r \C^n \dfn \{\bT \in (\C^n)^{\otimes r} \mid \pi(\bT) = \bT \text{ for all } \pi \in S_r\},
\]
where $S_r$ is the symmetric group and $\pi \in S_r$ acts by permuting the tensor factors.
The symmetric subspace has dimension $n[r] \dfn \binom{n+r-1}{r}$.

The symmetric subspace can equivalently be defined as the span of all symmetric rank-1 tensors, as shown by the following lemma.
This lemma holds over any vector space, but we prove it for $\C^n$ since this is the only case we will use.

\begin{lemma} \label{lemma:symmetric-span}
    The symmetric subspace is spanned by the set of rank-1 symmetric tensors.
\end{lemma}
\begin{proof}
    It suffices to show that, for any $\bv_1, \dots, \bv_r \in \C^n$, the symmetrization of $\bv_1 \otimes \dots \otimes \bv_r$ can be written as a linear combination of rank-1 symmetric tensors.
    This can be done using a Fourier transform:
    \[
        \sum_{\pi \in S_r} \bv_{\pi(1)} \otimes \dots \otimes \bv_{\pi(r)} = \E_{a \in \{-1,1\}^r}[(a_1 \dots a_r) (a_1 \bv_1 + \dots + a_r \bv_r)^{\otimes r}].
    \]
    To see why this equation holds, simply expand $(a_1 \bv_1 + \dots + a_r \bv_r)^{\otimes r}$ and observe that the only terms that survive are the ones which contain every $a_i$.
\end{proof}

In the case of $\C^n$, it turns out that symmetric rank-1 tensors effectively capture the linear structure of $\vee^r \C^n$ in a very strong sense.
The essential fact is Lemma~\ref{lemma:irrep}, which does \emph{not} hold over arbitrary vector spaces (for instance, it fails for $\R^n$).

\begin{lemma} \label{lemma:irrep}
    Let $\bv \in \C^n$ be any unit vector.
    If $U$ is a random unitary such that $U \bv = \bv$, then $\E U^{\otimes r} = (\bv \bv^\dagger)^{\otimes r}$ acts on the symmetric subspace as the projection onto $\bv^{\otimes r}$.
\end{lemma}
\begin{proof}
    Given any vector $\bw \in \C^n$, write $\bw = \bw_{\parallel} + \bw_{\perp}$ for its decomposition into components parallel and orthogonal to $\bv$.
    Letting $\lambda$ be a random complex number of modulus 1,
    \begin{align*}
        \E U^{\otimes r} \bw^{\otimes r} &= \E U^{\otimes r} (\bw_{\parallel} + \bw_{\perp})^{\otimes r} \\
        &= \E_U U^{\otimes r} \E_{\lambda} (\bw_{\parallel} + \lambda \bw_{\perp})^{\otimes r}.
    \end{align*}
    Since $\E_{\lambda} \lambda^k = 0$ for any $k > 0$, we have $\E_{\lambda} (\bw_{\parallel} + \lambda \bw_{\perp})^{\otimes r} = \bw_{\parallel}^{\otimes r}$.
    This determines the action of $\E U^{\otimes r}$ on any $\bw^{\otimes r}$; by Lemma~\ref{lemma:symmetric-span} it follows that $\E U^{\otimes r} = (\bv \bv^{\dagger})^{\otimes r}$ on the symmetric subspace.
\end{proof}

Letting $\sym^{(n,r)}$ be the orthogonal projection onto the symmetric subspace of $(\C^n)^{\otimes r}$, we have the following important fact.

\begin{theorem} \label{theorem:symmetric-subspace}
    If $\bpsi \in \C^n$ is a random unit vector, then
    \[
        \E[(\bpsi \bpsi^\dagger)^{\otimes r}] = \frac{1}{n[r]} \, \sym^{(n,r)}.
    \]
\end{theorem}
\begin{proof}
    Let $\bT = \E (\bpsi \bpsi^{\dagger})^{\otimes r}$.
    First observe that $\bT$ commutes with $U^{\otimes r}$ for every unitary $U$.
    Therefore if $\bv \in \C^n$ is any unit vector and $U$ is a random unitary such that $U \bv = \bv$, Lemma~\ref{lemma:irrep} implies that
    \[
        \bT \bv^{\otimes r} = \E_U \bT U^{\otimes r} \bv^{\otimes r} = \E_U U^{\otimes r} \bT \bv^{\otimes r} = \bv^{\otimes r} \cdot \inner{\bv^{\otimes r}}{\bT \bv^{\otimes r}}.
    \]
    Furthermore $\inner{\bv^{\otimes r}}{\bT \bv^{\otimes r}}$ has the same value for every unit vector $\bv$ (again because $\bT$ commutes with every $U^{\otimes r}$).
    Let $c = \inner{\bv^{\otimes r}}{\bT \bv^{\otimes r}}$ be this value.
    Since vectors of the form $\bv^{\otimes r}$ span the symmetric subspace by Lemma~\ref{lemma:symmetric-span}, and $\bT$ is 0 on the subspace orthogonal to the symmetric subspace, we must have $\bT = c \, \sym^{(n,r)}$.

    It only remains to compute $c$.
    This can be done by taking the trace of both sides of the equation $\bT = c \, \sym^{(n,r)}$:
    \[
        \Tr(\bT) = \E \Tr((\bpsi \bpsi^{\dagger})^{\otimes r}) = \E[\Tr(\bpsi \bpsi^{\dagger})^r] = 1
    \]
    and $\Tr(c \, \sym^{(n,r)}) = c \, n[r]$.
\end{proof}

Theorem~\ref{theorem:symmetric-subspace} gives us a convenient way to sketch symmetric tensors, and is the only result we will need about the symmetric subspace.
The utility of Theorem~\ref{theorem:symmetric-subspace} is captured by the following statement, which is a corollary of the theorem.
We state it as a ``lemma'' because it is treated as such for the rest of the paper.

\begin{lemma} \label{lemma:symmetric-subspace-estimator}
    Let $\bT \in (\C^n)^{\otimes r}$ be a symmetric tensor and let $\bx \in \C^n$ be a vector.
    Let $\xi \dfn n[r] \cdot \inner{\bT}{\bpsi^{\otimes r}} \cdot \inner{\bpsi}{\bx}^r$ where $\bpsi \in \C^n$ is a random unit vector.
    Then $\E[\xi] = \inner{\bT}{\bx^{\otimes r}}$ and $\Var[\xi] \le 4^r \cdot \fnorm{\bT}^2 \cdot \norm{\bx}^{2r}$.
\end{lemma}
\begin{proof}
    That $\E[\xi] = \inner{\bT}{\bx^{\otimes r}}$ follows immediately from Theorem~\ref{theorem:symmetric-subspace}.
    For the variance,
    \begin{align*}
        \Var[\xi] &\le \E \abs{\xi}^2 \\
        &= n[r]^2 \cdot \E \abs{\inner{\bT}{\bpsi^{\otimes r}}}^2 \cdot \abs{\inner{\bpsi^{\otimes r}}{\bx^{\otimes r}}}^{2} \\
        &= n[r]^2 \cdot \E\, \inner{\bT \otimes \bx^{\otimes r}}{\bpsi^{\otimes 2r}} \inner{\bpsi^{\otimes 2r}}{\bT \otimes \bx^{\otimes r}} \\
        &= \frac{n[r]^2}{n[2r]} \inner{\bT \otimes \bx^{\otimes r}}{\sym^{(n,2r)} (\bT \otimes \bx^{\otimes r})} \\
        &\le \frac{n[r]^2}{n[2r]} \cdot \fnorm{\bT}^2 \cdot \norm{\bx}^{2r}
    \end{align*}
    where the last equality follows again from Theorem~\ref{theorem:symmetric-subspace}.
    Finally, we complete the proof by observing that $n[r] = \binom{n+r-1}{r} = n (n+1) \cdots (n+r-1) / r!$ and
    \begin{equation*}
        \frac{n[r]^2}{n[2r]} = \frac{(2r)!}{(r!)^2} \cdot \frac{(n (n+1) \cdots (n+r-1))^2}{n (n+1) \cdots (n+2r-1)} \le \frac{(2r)!}{(r!)^2} = \binom{2r}{r} \le 4^r.
        \tag*{\qedhere}
    \end{equation*}
\end{proof}

Interestingly, if we had instead used a random \emph{real} unit vector $\bpsi \in \R^n$, then the conclusions of Theorem~\ref{theorem:symmetric-subspace} and Lemma~\ref{lemma:symmetric-subspace-estimator} would not hold.
Even for $n = r = 2$, there does not exist a distribution over real vectors $\bpsi \in \R^n$ such that $\E[(\bpsi \bpsi^T)^{\otimes r}]$ is the projection onto the symmetric subspace of $(\R^n)^{\otimes r}$.
To see why, let $\{\be_1, \be_2\}$ be an orthonormal basis of $\R^2$ and consider the quantity $(\be_1 \otimes \be_1)^T \E[(\bpsi \bpsi^T)^{\otimes 2}] (\be_2 \otimes \be_2) = \E[\inner{\be_1}{\bpsi}^2 \cdot \inner{\be_2}{\bpsi}^2]$.
If $\E[(\bpsi \bpsi^T)^{\otimes r}]$ is the projection onto the symmetric subspace, then this quantity should be 0, but for real distributions this forces one of the coordinates of $\bpsi$ to be 0 with probability 1.

% !TEX root = main.tex
\subsection{Stable arithmetic circuits} \label{subsection:stable-circuits}

Our stability condition is closely related to, and motivated by, the modulus of continuity.

\begin{definition}
    Let $f : \C^n \to \C$ and $\alpha : \R \to \R \cup \{\infty\}$ be functions.
    We say that a function $f$ has \emph{modulus of continuity} $\alpha$ at a point $\bx_* \in \C^n$ if
    \[
        \abs{f(\bx) - f(\bx_*)} \le \alpha(\norm{\bx - \bx_*})
    \]
    for all inputs $\bx$.
\end{definition}

Cauchy's estimate implies that the modulus of continuity is related to a bound on the decay of Taylor coefficients, as explained in the following fact.
This fact will not be used in any of our proofs, but it guides our definition of stability.

\begin{fact*}
    Let $f$ be an arithmetic circuit with $n$ inputs.
    Let $\bT^{(r)}$ be the $r$th Taylor coefficient of $f$, so that $f(\bx) = \sum_{r \ge 0} \inner{\bT^{(r)}}{\bx^{\otimes r}}$.
    If $f$ has modulus of continuity $\alpha$ at $\bzero$, then for all $B > 0$ and all integers $r \ge 1$,
    \begin{equation} \label{inequality:modulus-of-continuity}
        \opnorm{\bT^{(r)}} \le \frac{\alpha(B)}{B^r}.
    \end{equation}
    Conversely, if Inequality~(\ref{inequality:modulus-of-continuity}) holds for all $B > 0$ and all integers $r \ge 1$, then for any $\beta > 1$, $f$ has modulus of continuity $\alpha'$ at $\bzero$ where $\alpha'(B) \dfn \alpha(\beta \, B) / (\beta - 1)$.
\end{fact*}
\begin{proof}
    For any unit vector $\bpsi \in \C^n$, let $g_{\bpsi} : \C \to \C$ be the function defined by $g_{\bpsi}(z) = f(z \bpsi)$.
    Then for all $z \in \C$, $\abs{g_{\bpsi}(z) - g_{\bpsi}(0)} \le \alpha(\abs{z})$ because $f$ has modulus of continuity $\alpha$ at $\bzero$.
    By Cauchy's estimate \cite[Theorem 10.26]{Rud87}, it follows that
    \[
        \frac{1}{r!} \, \abs{\left.\frac{d^r g_{\bpsi}(z)}{dz^r}\right|_{z=0}} = \abs{\inner{\bT^{(r)}}{\bpsi^{\otimes r}}} \le \frac{\alpha(B)}{B^r}
    \]
    for all $B > 0$ and all integers $r \ge 1$.
    Inequality~(\ref{inequality:modulus-of-continuity}) follows by taking the maximum over all unit vectors $\bpsi \in \C^n$ and using the fact that $\opnorm{\bT} = \sup_{\bpsi \in \C^n, \norm{\bpsi}=1} \abs{\inner{\bT}{\bpsi^{\otimes r}}}$ for any symmetric tensor $\bT \in (\C^n)^{\otimes r}$ \cite[Lemma 1(6)]{symmetric-tensor-norm}.

    For the other direction, suppose that Inequality~(\ref{inequality:modulus-of-continuity}) holds for all $B > 0$ and all integers $r \ge 1$.
    Then for all $\bx \in \C^n$,
    \begin{align*}
        \abs{f(\bx) - f(\bzero)} &= \abs{\sum_{r \ge 1} \inner{\bT^{(r)}}{\bx^{\otimes r}}} \\
        &\le \sum_{r \ge 1} \norm{\bx}^r \, \opnorm{\bT^{(r)}} \\
        &= \sum_{r \ge 1} \beta^{-r} \, (\beta \norm{\bx})^r \opnorm{\bT^{(r)}} \\
        &\le \sum_{r \ge 1} \beta^{-r} \, \alpha(\beta \norm{\bx}) \\
        &= \frac{1}{\beta-1} \, \alpha(\beta \, \norm{\bx})
    \end{align*}
    where we have used Inequality~(\ref{inequality:modulus-of-continuity}) for the second inequality.
\end{proof}

Unfortunately, it is not enough that a function have bounded modulus of continuity for the method of sketching in this paper to work.
We need the following stronger condition.

\begin{definition}
    Let $f$ be an arithmetic circuit with $n$ inputs.
    Let $\bT^{(r)}$ be the $r$th Taylor coefficient of $f$, so that $f(\bx) = \sum_{r \ge 0} \inner{\bT^{(r)}}{\bx^{\otimes r}}$.
    We say that $f$ is \emph{$\alpha$-stable} if
    \[
        \fnorm{\bT^{(r)}} \le \frac{\alpha(B)}{B^r}
    \]
    for every $B > 0$ and every integer $r \ge 1$.
\end{definition}

Note that the stability condition is identical to Inequality~(\ref{inequality:modulus-of-continuity}), except that it uses the Frobenius norm instead of the operator norm.
Stability is therefore a significantly stronger condition: In general the only relationship between Frobenius and operator norms for a symmetric tensor $\bT \in (\C^n)^{\otimes r}$ is
\[
    \opnorm{\bT} \le \fnorm{\bT} \le \sqrt{n[r]} \, \opnorm{\bT}.
\]
Stability can be viewed as a combination of a bound on the modulus of continuity and a bound on the stable tensor rank $\fnorm{\bT}^2 / \opnorm{\bT}^2$ of the Taylor coefficient tensors.

% !TEX root = main.tex
\section{Data deletion} \label{section:data-deletion}

This section contains the main ideas and results of this paper.
We begin by explaining the data deletion problem in Section~\ref{subsection:data-deletion-problem}.
We then present our data deletion scheme in Section~\ref{subsection:data-deletion-scheme}.
We conclude with our experiments on microgpt in Section~\ref{subsection:microgpt}.

% !TEX root = main.tex
\subsection{The data deletion problem} \label{subsection:data-deletion-problem}
A data deletion scheme allows us to quickly predict model behavior after deleting training data.
It consists of two phases:
\begin{enumerate}
    \item \emph{precomputation}, which is allowed to expend resources comparable to those required to train the entire model; and
    \item \emph{prediction}, which should use the precomputed data to perform inference as if the model was trained without a given subset of training data.
\end{enumerate}
Using random access to the precomputed data, prediction should run in time much faster than re-training.
Ideally, it should run in time independent of the number of training examples $n$.

We assume that training and inference are deterministic, which is without loss of generality as one can always fix the randomness.
See \cite[Section 2.1]{magic} for a discussion on the benefits of this approach, which they refer to as ``single-model'' predictive data attribution.
We adopt a similar terminology and setting as in that paper, although crucially we do \emph{not} assume that $\phi$ is known during precomputation.

\begin{itemize}
    \item There are $n$ training examples and the model has $p$ parameters.
    \item $\calA$ is a \emph{learning algorithm} specified by an arithmetic circuit with $n$ inputs and $p$ outputs.
    The input to $\calA$ is a vector of \emph{downweights} $\bw \in [0,1]^n$ where $w_i$ specifies an amount to ``remove'' example $i$.
    For instance, if the learning algorithm is a variation of gradient descent then we can multiply the loss on example $i$ by $1-w_i$ before backpropagation (see Section~\ref{subsection:technical-overview}).
    When $\bw = \bzero$, the model should be trained as usual; when $w_i = 1$, $\calA(\bw)$ should contain no information about example $i$.
    \item $\phi$ is a \emph{measurement function} specified by an arithmetic circuit with $p$ inputs and 1 output.
    \item $f$ is the composition $\phi \circ \calA$.
    It is an arithmetic circuit with $n$ inputs and 1 output.
\end{itemize}

% !TEX root = main.tex
\subsection{Our data deletion scheme} \label{subsection:data-deletion-scheme}

Our data deletion scheme is given in Algorithms~\ref{algorithm:precomputation} and~\ref{algorithm:prediction}.
The correctness of our scheme relies only on the assumption that $f = \phi \circ \calA$ is a stable arithmetic circuit in the sense of Section~\ref{subsection:stable-circuits}.
Note that, while precomputation does not know which measurement $\phi$ will ultimately be used during prediction, the accuracy of the prediction nonetheless does not require any condition on $\calA$ directly.
The prediction will be accurate for any $\phi$ such that $\phi \circ \calA$ is stable, even if $\phi$ or $\calA$ alone is highly unstable.

In order to achieve exponential concentration, our prediction algorithm will make use of the median-of-means estimator.
We briefly recall how the median-of-means works here.

For $\xi_1, \dots, \xi_{k} \in \C$, let $\mom_m(\xi_1, \dots, \xi_k)$ be the following algorithm.
\begin{enumerate}
    \item Partition $[k]$ into $m$ equal-sized blocks $B_1, \dots, B_m$.
    \item Compute $\eta_i \dfn \frac{1}{\abs{B_i}} \sum_{j \in B_i} \xi_j$ for each $i \in [m]$.
    \item Output $\median(\{\Re(\eta_i)\}_{i \in [m]}) + \ii \cdot \median(\{\Im(\eta_i)\}_{i \in [m]})$.
\end{enumerate}

Note that $\mom_m(\xi_1, \dots, \xi_k)$ can be computed in $O(k)$ operations.

\begin{lemma} \label{lemma:mom}
    If $\calD$ is a distribution on $\C$ with mean $\mu$ and finite variance $\sigma^2$, then
    \[
        \Pr_{\xi_1, \dots, \xi_{k} \sim \calD}\!\left[\abs{\mom_m(\xi_1, \dots, \xi_{k}) - \mu} > \sigma \, \sqrt{\frac{4 m}{k}}\,\right] \le 2 e^{-m/8}
    \]
\end{lemma}
\begin{proof}
    Let $\sigma_{\Re}^2$ be the variance of $\Re(\xi)$ and $\sigma_{\Im}^2$ be the variance of $\Im(\xi)$, for $\xi \sim \calD$.
    A standard argument (see, for instance, \cite[Proposition 12]{lerasle}) implies that
    \begin{align*}
        \Pr_{\xi_1, \dots, \xi_{k} \sim \calD}\!\left[\abs{\mom_m(\Re(\xi_1), \dots, \Re(\xi_{k})) - \Re(\mu)} > 2 \sigma_{\Re} \, \sqrt{\frac{m}{k}}\,\right] &\le e^{-m/8} \text{ and} \\
        \Pr_{\xi_1, \dots, \xi_{k} \sim \calD}\!\left[\abs{\mom_m(\Im(\xi_1), \dots, \Im(\xi_{k})) - \Im(\mu)} > 2 \sigma_{\Im} \, \sqrt{\frac{m}{k}}\,\right] &\le e^{-m/8}.
    \end{align*}
    Since $\Re(\mom_m(\xi_1, \dots, \xi_{k})) = \mom_m(\Re(\xi_1), \dots, \Re(\xi_{k}))$ and $\Im(\mom_m(\xi_1, \dots, \xi_{k})) = \mom_m(\Im(\xi_1), \dots, \Im(\xi_{k}))$, by a union bound we have
    \[
        \Pr_{\xi_1, \dots, \xi_{k} \sim \calD}\!\left[\abs{\mom_m(\xi_1, \dots, \xi_{k}) - \mu} > 2 \sqrt{\sigma_{\Re}^2 + \sigma_{\Im}^2} \, \sqrt{\frac{m}{k}}\,\right] \le 2 e^{-m/8}
    \]
    as desired.
\end{proof}

\begin{algorithm}[H]
\caption{$\precompute_{s,k}(\calA)$: Precomputation} \label{algorithm:precomputation}
    \KwData{Integers $s, k \ge 1$, where $s$ is used to specify a ring $R \dfn \C[z] / (z^{s+1})$.}
    \KwIn{Learning algorithm $\calA$ specified as an arithmetic circuit with $n$ inputs and $p$ outputs.}
    \KwOut{A list of directions $ \bPsi = \{\bpsi_i\}_{i \in [k]} \subseteq \C^n$ and Taylor approximation data $\bP = \{\bp_i\}_{i \in [k]} \subseteq R^p$.}
    \AlgoMetaSep
    Sample random unit vectors $\bpsi_1, \dots, \bpsi_k \in \C^n$\;
    For each $i \in [k]$, compute $\bp_i \gets \calA(z \, \bpsi_i)$ over $R$\;
    \Return $\bPsi \dfn \{\bpsi_i\}_{i \in [k]}$ and $\bP \dfn \{\bp_i\}_{i \in [k]}$\;
\end{algorithm}

We emphasize that $\precompute$ (Algorithm~\ref{algorithm:precomputation}) does not require knowledge of the measurement function $\phi$; the same precomputed data can be used for arbitrary measurements.
We also note that it is not necessary to actually store $\bPsi$, if we instead use a pseudorandom function to select the components $\psi_{i,j}$ of the directions $\bpsi_i$ as a function of $i \in [k]$ and $j \in [n]$.

\begin{algorithm}[H]
\caption{$\predict_{\bPsi, \bP, m}(D, \phi)$: Counterfactual prediction} \label{algorithm:prediction}
    \KwData{Directions $\bPsi = \{\bpsi_i\}_{i \in [k]} \subseteq \C^n$, Taylor approximation data $\bP = \{\bp_i\}_{i \in [k]} \subseteq R^p$ where $R \dfn \C[z]/(z^{s+1})$, and integer $m \ge 1$.}
    \KwIn{Deletion set $D \subseteq [n]$ and measurement function $\phi$ specified as an arithmetic circuit with $p$ inputs and 1 output.}
    \KwOut{Estimate $\nu$ for post-deletion measurement $(\phi \circ \calA)(\bone_D)$.}
    \AlgoMetaSep
    For each $i \in [k]$, compute $q_i \gets \phi(\bp_i)$ over $R$ and write
    \[
        q_i = q_{i,0} + q_{i,1} z + \dots + q_{i,s} z^s
    \]
    where $q_{i,0}, \dots, q_{i,s} \in \C$\;
    For each $r \in \{0, \dots, s\}$, compute
    \[
        \nu_r \gets n[r] \cdot \mom_m\left(\left\{\inner{\bpsi_i}{\bone_D}^r \, q_{i,r}\right\}_{i \in [k]}\right)
    \]
    where $\inner{\bpsi_i}{\bone_D}$ is evaluated by looking up the $D$ coordinates of $\bpsi_i$\;
    Compute $\nu \gets \nu_0 + \dots + \nu_s$\;
    \Return $\nu$\;
\end{algorithm}

The following theorem shows that Algorithms~\ref{algorithm:precomputation} and~\ref{algorithm:prediction} form an accurate and efficient data deletion scheme, assuming that $\phi \circ \calA$ satisfies the stability condition explained in Section~\ref{subsection:stable-circuits}.

\begin{theorem} \label{theorem:deletion}
    For any $\varepsilon, \delta > 0$, there exists a choice of integers $s, k, m \ge 1$ in Algorithms~\ref{algorithm:precomputation} and~\ref{algorithm:prediction} such that, for any learning algorithm $\calA$ taking $n$ training examples to $p$ parameters,
    \begin{itemize}
        \item precomputation requires $\tilde{O}(\size(\calA) \, \log(1/\delta) / \varepsilon^2)$ operations,
        \item the required storage is $\tilde{O}((n + p) \, \log(1/\delta) / \varepsilon^2)$, and
    \end{itemize}
    for any measurement function $\phi$ taking $p$ parameters to 1 outcome and any set $D \subseteq [n]$ of up to $d$ examples,
    \begin{itemize}
        \item predicting $(\phi \circ \calA)(\bone_D)$ requires $\tilde{O}((\size(\phi) + d) \, \log(1/\delta) / \varepsilon^2)$ operations, and
        \item if $\phi \circ \calA$ is $\alpha$-stable then the prediction $\nu$ has error $\abs{\nu - (\phi \circ \calA)(\bone_D)} \le \varepsilon \cdot \alpha(4 \sqrt{d})$ with probability $1 - \delta$,
    \end{itemize}
    where $\tilde{O}$ hides $\polylog(1/\varepsilon)$ factors.
\end{theorem}
\begin{proof}
    We start by bounding the error, then we analyze the resource requirements.

    \paragraph{The error.}
    Let $f \dfn \phi \circ \calA$ and write the Taylor series of $f$ as
    \[
        f(\bx) = \sum_{r \ge 0} \inner{\bT^{(r)}}{\bx^{\otimes r}}.
    \]
    First observe that $q_i = f(z \bpsi_i)$, so $q_{i,r} = \inner{\bT^{(r)}}{\bpsi_i^{\otimes r}}$ for all $r \le s$.    
    We will bound the error in two components: the error from truncating to order $s$ and the error from the variance of the sketch.
    
    The error from truncating $f$ to order $s$ is at most
    \begin{align*}
        \abs{\sum_{r=s+1}^{\infty} \inner{\bT^{(r)}}{\bone_D^{\otimes r}}}
        &\le \sum_{r=s+1}^{\infty} \fnorm{\bT^{(r)}} \, \norm{\bone_D}^r \\
        &= \sum_{r=s+1}^{\infty} \frac{1}{4^r} \, \fnorm{\bT^{(r)}} \, \left(4 \sqrt{d}\right)^r \\
        &\le \sum_{r=s+1}^{\infty} \frac{1}{4^r} \, \alpha\!\left(4 \sqrt{d}\right) \\
        &\le \frac{1}{4^s} \, \alpha\!\left(4 \sqrt{d}\right).
    \end{align*}
    We now turn to the error from the variance of the sketch.
    Note that we always have $\nu_0 = f(\bzero)$ because $q_{i,0} = f(\bzero)$ deterministically for every $i \in [k]$.
    For each $r \in [s]$, Lemma~\ref{lemma:symmetric-subspace-estimator} implies that
    \begin{align*}
        \E[n[r] \, q_{i,r} \inner{\bpsi_i}{\bone_D}^r] &= \E\left[n[r] \, \inner{\bT^{(r)}}{\bpsi_i^{\otimes r}} \cdot \inner{\bpsi_i}{\bone_D}^r\right] \\
        &= \inner{\bT^{(r)}}{\bone_D^{\otimes r}}
    \end{align*}
    and
    \[
        \Var[n[r] \, q_{i,r} \inner{\bpsi_i}{\bone_D}^r] \le (4 d)^r \, \fnorm{\bT^{(r)}}^2.
    \]
    By Lemma~\ref{lemma:mom}, it follows that each median-of-means estimate $\nu_r$ satisfies
    \[
        \abs{\nu_r - \inner{\bT^{(r)}}{\bone_D^{\otimes r}}} \le (4 d)^{r/2} \sqrt{\frac{4 m}{k}} \, \fnorm{\bT^{(r)}}
    \]
    with probability $1 - 2 \, e^{-m/8}$.
    By a union bound, with probability $1 - 2 s \, e^{-m/8}$ the final error up to order $s$ is
    \begin{align*}
        \abs{\sum_{r=1}^s \nu_r - \sum_{r=1}^s \inner{\bT^{(r)}}{\bone_D^{\otimes r}}}
        &\le \sum_{r=1}^s (4 d)^{r/2} \sqrt{\frac{4 m}{k}} \, \fnorm{\bT^{(r)}} \\
        &= \sum_{r=1}^s \frac{1}{2^r} \, \sqrt{\frac{4 m}{k}} \, \fnorm{\bT^{(r)}} \, \left(4 \sqrt{d}\right)^r \\
        &\le \sum_{r=1}^s \frac{1}{2^r} \, \sqrt{\frac{4 m}{k}} \, \alpha\!\left(4 \sqrt{d}\right) \\
        &\le \sqrt{\frac{4 m}{k}} \, \alpha\!\left(4 \sqrt{d}\right).
    \end{align*}
    Adding this to the truncation error, it follows that
    \[
        \abs{\nu - f(\bone_D)} \le \left(\frac{1}{4^s} + \sqrt{\frac{4 m}{k}}\right) \cdot \alpha\!\left(4 \sqrt{d}\right)
    \]
    with probability $1 - 2 s \, e^{-m/8}$.
    We get the desired bounds in terms of $\varepsilon, \delta$ by setting
    \begin{itemize}
        \item $s = \log_4(2/\varepsilon) = \tilde{O}(1)$,
        \item $m = 8 \ln(2 s / \delta) = \tilde{O}(\log(1/\delta))$, and
        \item $k = 16 m / \varepsilon^2 = \tilde{O}(\log(1/\delta)/\varepsilon^2)$.
    \end{itemize}

    \paragraph{The complexity of precomputation and storage requirements.}
    Algorithm~\ref{algorithm:precomputation} just computes $k$ evaluations of $\calA$ over $R$.
    So following the discussion in Section~\ref{subsection:forward-ad}, its complexity is $O(\size(\calA) \, k s \log s) = \tilde{O}(\size(\calA) \, \log(1/\delta)/\varepsilon^2)$.

    The storage requirements are $O(k n)$ for $\bPsi$ and $O(k s p)$ for $\bP$, for a total of $O(k n + k s p) = \tilde{O}((n + p) \, \log(1/\delta)/\varepsilon^2)$.
    This can be reduced to $O(k s p) = \tilde{O}(p \, \log(1/\delta)/\varepsilon^2)$ if we select the $j$th coordinate of $\bpsi_i$ as a pseudorandom function of $(i, j)$.

    \paragraph{The complexity of prediction.}
    In Algorithm~\ref{algorithm:prediction}, we first compute $k$ evaluations of $\phi$ over $R$.
    This requires $O(\size(\phi) \, k s \log s)$ operations.

    Then we compute $\inner{\bpsi_i}{\bone_D} = \sum_{j \in D} \overline{\psi}_{i,j}$ for all $i \in [k]$ in time $O(d k)$.
    Each $\nu_r$ then requires computing a median-of-means in time $O(k)$.
    Finally, we compute $\nu \gets \nu_0 + \dots + \nu_s$ in time $O(s)$.

    In total, Algorithm~\ref{algorithm:prediction} requires $O(\size(\phi) \, k s \log s + d k + s k + s) = \tilde{O}( (\size(\phi) + d) \, \log(1/\delta) / \varepsilon^2)$ operations.
\end{proof}

It is important for us to use \emph{complex} directions $\bpsi$ because there is no replacement for Theorem~\ref{theorem:symmetric-subspace} (and by extension Lemma~\ref{lemma:symmetric-subspace-estimator}) that uses a distribution over real vectors.
See the discussion at the end of Section~\ref{subsection:symmetric-subspace} for an explanation.
In order to avoid complex directions, it appears to be necessary to use multi-directional derivatives.
Whereas our algorithms run in time almost linear in $s$, using multi-directional derivatives would require running in time exponential in $s$.

% !TEX root = main.tex
\subsection{Experiments: the stability of microgpt} \label{subsection:microgpt}

Theorem~\ref{theorem:deletion} fully specifies the performance of our data deletion method up to the stability function $\alpha$.
In this section we investigate this function for microgpt \cite{microgpt}, finding that it appears to behave favorably for our method.
The experiments in this section are extremely minimal, and intended as a proof of concept.
We hope that future work can better understand the scaling of $\alpha$ with the size of the model and amount of training data.

All of our experiments are done with the same set-up as microgpt \cite{microgpt}, except that we increase $\texttt{eps\_adam}$ from $10^{-8}$ to $10^{-3}$ and move it inside the square root (a similar modification was used in \cite{replay,magic}), and we replace ReLU gates with GELU gates.
The code is available at \href{https://github.com/SamSpo1/microgpt-sketch}{https://github.com/SamSpo1/microgpt-sketch}.

Let $f = \phi \circ \calA$ be the composition of the learning algorithm and measurement function and write
\[
    f(\bx) = \sum_{r \ge 0} \inner{\bT^{(r)}}{\bx^{\otimes r}}.
\]
Theorem~\ref{theorem:deletion} requires a bound on the stability $\alpha$ of $f$, which (recall Section~\ref{subsection:stable-circuits}) controls the Frobenius norm of $\bT^{(r)}$:
\[
    \fnorm{\bT^{(r)}} \le \frac{\alpha(B)}{B^r}
\]
for all $B > 0$.
Given that $\bT^{(r)}$ has $n^r$ coordinates, it can be prohibitively expensive to compute $\fnorm{\bT^{(r)}}$ directly.
Instead we use Lemma~\ref{lemma:symmetric-subspace-estimator}, which tells us that if $\bpsi \in \C^n$ is a random unit vector then\footnote{Technically, Lemma~\ref{lemma:symmetric-subspace-estimator} only applies for estimating $\inner{\bT}{\bx^{\otimes r}}$. However the exact same argument applies for estimating $\inner{\bT}{\bT'}$.}
\[
    n[r] \, \E_{\bpsi}\left[\abs{\inner{\bT^{(r)}}{\bpsi^{\otimes r}}}^2\right] = \fnorm{\bT^{(r)}}^2
\]
and the variance is at most $4^r \fnorm{\bT^{(r)}}^4$.
Motivated by this, we estimate $\fnorm{\bT^{(r)}}$ by computing
\[
    \sqrt{n[r]} \, \abs{\inner{\bT^{(r)}}{\bpsi^{\otimes r}}}
\]
for several random unit vectors $\bpsi \in \C^n$ and $r \in [100]$.
The results are presented in Figure~\ref{figure-1} on a logarithmic scale.

\begin{figure}[ht]
    \centering
    \input{fig1.tex}
    \caption{Estimates of $\log \fnorm{\bT^{(r)}}$, i.e., the log Frobenius norms of the Taylor coefficients of $f = \phi \circ \calA$.
    The estimates are obtained as $\log\left(\sqrt{n[r]} \, \abs{\inner{\bT^{(r)}}{\bpsi^{\otimes r}}}\right)$ for random complex directions $\bpsi$.
    The learning algorithm $\calA$ trains microgpt with Adam for 1000 steps on a dataset of names with a batch size of 1, as in \cite{microgpt}.
    Here $\phi$ is the loss of the given model on a random training example (chosen independently for each direction $\bpsi$).}
    \label{figure-1}
\end{figure}

Let us briefly explain how the decay of $\fnorm{\bT^{(r)}}$ translates to $\alpha$.
If
\[
    \fnorm{\bT^{(r)}} = 2^{-\omega(r)},
\]
then for any $B > 0$, we will have a finite bound $\alpha(B) < \infty$.
In particular, for any constant $d$, Theorem~\ref{theorem:deletion} will guarantee an error of $O(\varepsilon)$.
Our experiments with microgpt in Figure~\ref{figure-1} suggest that in fact we have something more like
\[
    \fnorm{\bT^{(r)}} = 2^{-\Omega(r^{3/2})}.
\]
These results show that our data deletion scheme should achieve vanishing error for at least small numbers of deletions.
This is borne out by Figures~\ref{figure-2} and~\ref{figure-3}, explained below.
We did not attempt to figure out how $\fnorm{\bT^{(r)}}$ scales with the number of training examples $n$ and the number of model parameters $p$, which could be important for determining the behavior for larger numbers of deletions.

We conclude with two figures describing the function $z \mapsto f(z \, \bone_D)$ for $z \in [0,1]$.
In Figure~\ref{figure-2}, we examine how this function compares to the \emph{exact} Taylor approximations around $z = 0$.
In Figure~\ref{figure-3}, we examine how it compares to the \emph{sketched} Taylor approximations obtained from our data deletion scheme.
Figure~\ref{figure-3} is therefore a demonstration that our method works in practice, at least for microgpt.
Together, Figures~\ref{figure-2} and~\ref{figure-3} show that higher-order approximations are both necessary and sufficient for achieving accurate predictions in our set-up.

\begin{remark}
    As explained in Section~\ref{subsection:related-work}, our method is already novel in the linear case, but higher-order approximations allow us to get the strong guarantees of Theorem~\ref{theorem:deletion}.
    Even if one modifies the set-up to ensure that the linear approximation is accurate, our method is still the only one to achieve a guaranteed bound on the prediction error without needing foreknowledge of the measurement function during precomputation.
\end{remark}

\begin{figure}[ht]
    \centering
    % !TEX root = main.tex
\begin{tikzpicture}
    \pgfmathsetmacro{\cA}{3.0913238794422018}
    \pgfmathsetmacro{\cB}{0.25965844733840343}
    \pgfmathsetmacro{\cC}{0.11163512679844428}
    \pgfmathsetmacro{\cD}{0.0679755958046495}
    \pgfmathsetmacro{\cE}{0.032059205918654306}
    \pgfmathsetmacro{\cF}{0.013569830861609827}

    \begin{axis}[
        width=12cm,
        height=8cm,
        xmin=0, xmax=1,
        ymin=3.0667406363510072, ymax=3.607346720434554,
        xlabel={downweight $z$},
        ylabel={$f(z \bone_D)$},
        grid=both,
        major grid style={draw=gray!30},
        minor grid style={draw=gray!15},
        tick align=outside,
        legend pos=north west,
        legend cell align=left,
        legend style={draw=none, fill=white, font=\small},
        line cap=round,
        line join=round,
    ]

    \addplot[
        black,
        thick,
        mark=*,
        only marks,
        mark size=1.6pt,
    ] coordinates {
        (0, 3.0913136401729866)
        (0.10000000000000001, 3.1184669086184376)
        (0.20000000000000001, 3.1483101734636838)
        (0.29999999999999999, 3.1813910249483568)
        (0.40000000000000002, 3.2183783688681764)
        (0.5, 3.2600910696458101)
        (0.59999999999999998, 3.307521541679395)
        (0.69999999999999996, 3.3618247240796313)
        (0.80000000000000004, 3.4242100405076288)
        (0.90000000000000002, 3.4956964339267986)
        (1, 3.5767521812165035)
    };
    \addlegendentry{empirical}

    \addplot[
        red,
        thick,
        dashed,
        domain=0:1,
        samples=300,
    ] {\cA + \cB*x};
    \addlegendentry{degree 1}

    \addplot[
        green!60!black,
        thick,
        dotted,
        domain=0:1,
        samples=300,
    ] {\cA + \cB*x + \cC*x^2};
    \addlegendentry{degree 2}

    \addplot[
        blue,
        thick,
        dashdotted,
        domain=0:1,
        samples=300,
    ] {\cA + \cB*x + \cC*x^2 + \cD*x^3};
    \addlegendentry{degree 3}

    \addplot[
        orange,
        thick,
        loosely dashed,
        domain=0:1,
        samples=300,
    ] {\cA + \cB*x + \cC*x^2 + \cD*x^3 + \cE*x^4};
    \addlegendentry{degree 4}

    \addplot[
        purple,
        thick,
        solid,
        domain=0:1,
        samples=300,
    ] {\cA + \cB*x + \cC*x^2 + \cD*x^3 + \cE*x^4 + \cF*x^5};
    \addlegendentry{degree 5}

    \end{axis}
\end{tikzpicture}
    \caption{Taylor approximations to $f(z \, \bone_D)$, where $f = \phi \circ \calA$.
    As in Figure~\ref{figure-1}, the learning algorithm $\calA$ trains microgpt for 1000 steps on a dataset of names \cite{microgpt}.
    The deleted set $D$ is all names containing the letter ``x'' (there are 18) and the measurement function $\phi$ is the loss on the name ``max.''
    The black dots are the values $f(z \, \bone_D)$ for $z \in \{0, 0.1, \dots, 1\}$, and the colored lines are Taylor approximations to the function $z \mapsto f(z \, \bone_D)$ around $z = 0$.}
    \label{figure-2}
\end{figure}

In Figure~\ref{figure-3}, we evaluate our actual data deletion scheme with $k = 500$ samples in the same set-up as Figure~\ref{figure-2}.
Interestingly, although the number of deletions is large enough that the variance grows with the degree, higher-order approximations nonetheless outperform the linear approximation.

\begin{figure}[ht]
    \centering
    \input{fig3.tex}
    \caption{Predicted loss obtained using Algorithms~\ref{algorithm:precomputation} and~\ref{algorithm:prediction} with the same $\calA$, $\phi$, and $D$ as in Figure~\ref{figure-2}.
    We use $s \in \{1,2,3\}$ and $k = 500$ in $\precompute_{s,k}(\calA)$ and $m = 1$ in $\predict_{\bPsi,\bP,m}(D, \phi)$.
    The lines are the predictions obtained by the data deletion scheme, with 1 standard deviation shaded.
    In this case the approximation to the degree 3 expansion happens to be more accurate than the degree 3 expansion itself (cf. Figure~\ref{figure-2}), but this is pure chance.}
    \label{figure-3}
\end{figure}

\bibliographystyle{alpha}
\bibliography{0-references}

\appendix

% !TEX root = main.tex
\section{Supplement to related work: post hoc unlearning} \label{section:post-hoc-unlearning}

A post hoc unlearning method does not make any reference to the algorithm used to train the model.
While post hoc unlearning methods may or may not work in practice, any security proof for a post hoc machine unlearning method must depend on the particular learning algorithm that was used.
We explain the situation below.

Let $\approx$ be any kind of indistinguishability (such as statistical distance, computational indistinguishability, or $(\varepsilon, \delta)$-differential privacy).
We say that a post hoc unlearning algorithm $\calU$ is $\approx$-certified unlearning for a learning algorithm $\calA$ if there exists an algorithm $\bar{\calA}$ such that, for any forget dataset $D$ and initial dataset $S$,
\[
    \calU(\calA(S), S, D) \approx \bar{\calA}(S \setminus D).
\]
This definition generalizes that of \cite{certified-unlearning}.
The following proposition shows that any proof that a scheme satisfies this definition of unlearning must depend on the learning algorithm.

\begin{proposition} \label{proposition:certified-unlearning}
    Suppose that $\calU$ is $\approx$-unlearning simultaneously for every learning algorithm $\calA$.
    Then for \emph{any} training algorithm $\calA$, \emph{any} non-empty forget set $D$, \emph{any} initial dataset $S$, and any $x_* \not\in S$,
    \[
        \calU(\calA(S), S, D) \approx_2 \calU(\bot, S \cup \{x_*\}, D \cup \{x_*\}),
    \]
    where $\approx_2$ denotes the indistinguishability resulting from two applications of $\approx$ and $\bot$ denotes a dummy model (e.g., one with all parameters set to 0).
\end{proposition}

In particular, Proposition~\ref{proposition:certified-unlearning} means that no post hoc unlearning scheme with a security proof that does not depend on the training algorithm can perform better than re-training: It does not even materially use the initial trained model $\calA(S)$.

\begin{proof}
    Let $x_*$ be any data point and $\calA$ be any learning algorithm.
    Define
    \[
        \calB(S) = \begin{cases}
            \bot & \text{if } x_* \in S, \\
            \calA(S) & \text{otherwise.}
        \end{cases}
    \]
    Since $\calU$ is $\approx$-unlearning for every algorithm, there exists $\bar{\calB}$ such that, for every forget set $D$ and every initial dataset $S$,
    \begin{equation} \label{equation:no-agnostic-deletion}
        \calU(\calB(S), S, D) \approx \bar{\calB}(S \setminus D).
    \end{equation}
    In particular, for any forget set $D$ and any dataset $S$ not containing $x_*$,
    \begin{align*}
        \calU(\bot, S \cup \{x_*\}, D \cup \{x_*\}) &= \calU(\calB(S \cup \{x_*\}), S \cup \{x_*\}, D \cup \{x_*\}) \\
        &\approx \bar{\calB}(S \setminus D) \\
        \intertext{and}
        \calU(\calA(S), S, D) &= \calU(\calB(S), S, D) \\
        &\approx \bar{\calB}(S \setminus D).
    \end{align*}
    The $=$'s follow from the definition of $\calB$ and the $\approx$'s follow from Equation~\ref{equation:no-agnostic-deletion}.
    Therefore
    \[
        \calU(\calA(S), S, D) \approx_2 \calU(\bot, S \cup \{x_*\}, D \cup \{x_*\}),
    \]
    completing the proof.
\end{proof}

The reason that Proposition~\ref{proposition:certified-unlearning} does not contradict prior work with security proofs that do not depend on the learning algorithm is that those works do not come with provable \emph{correctness} guarantees.
Therefore, in the regime where the security proof applies, they can be viewed as essentially re-training in disguise.
This highlights the danger of using certified unlearning as a privacy guarantee without also analyzing correctness.

\end{document}